\documentclass[preprint,12pt]{elsarticle}

\usepackage{amssymb}
\usepackage{mathrsfs}
\usepackage{mathrsfs}
\usepackage{mathrsfs}
\usepackage{amsfonts}
\usepackage{amsfonts}
\usepackage{amsfonts}
\usepackage{mathrsfs}
\usepackage{amsmath}
\usepackage{graphicx}
\usepackage{multirow}
\usepackage{caption2}
\usepackage{float}
\usepackage{booktabs}
\usepackage{algorithm}
\usepackage{algorithmic}
\usepackage{subfigure}
\newtheorem{theorem}{Theorem}[section]

\numberwithin{equation}{section}

\newtheorem{lemma}[theorem]{Lemma}

\newtheorem{proposition}[theorem]{Proposition}
\newtheorem{remark}[theorem]{Remark}

\newenvironment{proof}[1][Proof]{\textbf{#1. }}{\ \rule{0.5em}{0.5em}}%
\graphicspath{{figures/}} \journal{}

\begin{document}

\begin{frontmatter}



\title{Nonparametric regression using needlet kernels for spherical data
\tnoteref{t1}} \tnotetext[t1]{The research was supported by the
National Natural Science Foundation of China (Grant Nos. 61502342,
11401462)}


\author{Shaobo Lin\corref{*}}\cortext[*]{Corresponding author: sblin1983@gmail.com}

\address{  College of Mathematics and Information Science, Wenzhou
University, Wenzhou 325035, China }

\begin{abstract}
Needlets have been recognized as  state-of-the-art tools to tackle
spherical data, due to  their excellent localization properties in
both  spacial and frequency domains.
 This paper  considers
developing kernel methods associated with the   needlet kernel for
nonparametric regression problems whose predictor variables are
defined on a sphere. Due to the localization property in the
frequency domain, we prove that the regularization parameter of the
 kernel ridge regression associated with the needlet kernel can
decrease arbitrarily fast. A natural consequence is that the
regularization term for the   kernel ridge regression is not
necessary in the sense of rate optimality. Based on the excellent
localization property in the  spacial domain  further, we also prove
that all the  $l^{q}$ $(0<q \leq 2)$ kernel regularization estimates
associated with the needlet kernel, including the kernel lasso
estimate and the kernel bridge estimate, possess almost the same
generalization capability for a large range of regularization
parameters in the sense of rate optimality.
 This finding tentatively reveals that, if the needlet kernel is utilized,
 then  the choice of $q$ might not have a strong impact in terms
of the generalization capability in some modeling contexts. From
this perspective, $q$ can be arbitrarily specified, or specified
merely by other no generalization criteria like smoothness,
computational complexity, sparsity, etc..
\end{abstract}

\end{frontmatter}

\section{Introduction}

Contemporary scientific investigations frequently encounter a common
issue of exploring the relationship between a response variable  and
a number of  predictor variables  whose domain is the surface of a
sphere. Examples include the study of gravitational phenomenon
\cite{Freeden1998}, cosmic microwave background radiation
\cite{Dolelson2003}, tectonic plat geology \cite{Chang2000}  and
image rendering \cite{Tsai2006}. As the sphere is topologically a
compact two-point homogeneous manifold, some widely used schemes for
the Euclidean space such as the neural networks \cite{Gyorfi2002}
and support vector machines \cite{Scholkopf2001} are no more the
most appropriate methods for tackling  spherical data. Designing
  efficient and exclusive approaches to extract useful information from
spherical data has been a recent focus in statistical learning
\cite{Downs2003,Marzio2014,Monnier2011,Pelletier2006}.

Recent years have witnessed   considerable approaches about
nonparametric regression for spherical data. A classical and
long-standing technique is the orthogonal series methods associated
with spherical harmonics \cite{Abrial2008}, with which the local
performance  of the estimate are quite poor, since spherical
harmonics are not well localized but spread out all over the sphere.
Another widely used technique is the stereographic projection
methods \cite{Downs2003}, in which the statistical problems on the
sphere were formulated in the Euclidean space by use of a
stereographic projection. A major problem is that  the stereographic
projection usually leads to a distorted theoretical analysis
paradigm and a relatively sophisticate  statistical behavior.
Localization methods, such as the Nadaraya-Watson-like estimate
\cite{Pelletier2006}, local polynomial estimate \cite{Bickel2007}
and local linear estimate \cite{Marzio2014} are also alternate and
interesting nonparametric approaches. Unfortunately, the manifold
structure of the sphere is not well taken into account in these
approaches. Mihn \cite{Minh2006} also developed a general theory of
reproducing kernel Hilbert space on the sphere and advocated to
utilize the kernel methods to tackle   spherical data. However, for
some popular kernels such as the Gaussian \cite{Minh2010} and
polynomials \cite{Cao2013}, kernel methods suffer from either  a
similar problem as the localization methods, or a similar drawback
as the orthogonal series methods. In fact, it remains open that
whether there is an exclusive kernel for spherical data such that
both the manifold structure of the sphere and the localization
requirement are sufficiently considered.

Our focus in this paper is not on developing a novel  technique to
cope with spherical nonparametric regression problems, but on
introducing an exclusive kernel for kernel methods. To be detailed,
we aim to find a kernel that possesses  excellent spacial
localization property and makes fully use of the manifold structure
of the sphere. Recalling that one of the most important factors to
embody the manifold structure is the  special frequency domain of
the sphere, a kernel which can control the frequency domain freely
is preferable. Thus, the kernel we need is   actually a function
that possesses excellent localization properties, both in  spacial
and frequency domains. Under this circumstance, the needlet kernel
comes into our sights. Needlets, introduced by Narcowich et al.
\cite{Narcowich2006,Narcowich20061}, are a new kind of
second-generation spherical wavelets, which can be shown to make up
a tight frame with both perfect spacial and frequency localization
properties. Furthermore, needlets have a clear statistical nature
\cite{Baldi2008,Kerkyacharian2012}, the most important of which is
that in the Gaussian and isotropic random fields, the random
spherical needlets behave asymptotically as an i.i.d. array
\cite{Baldi2008}. It can be found in \cite{Narcowich2006} that the
spherical needlets correspond a needlet kernel, which is also well
localized in the spacial  and frequency domains. Consequently, the
needlet kernel is proved to possess  the reproducing property
\cite[Lemma 3.8]{Narcowich2006}, compressible property \cite[Theorem
3.7]{Narcowich2006} and best approximation property \cite[Corollary
3.10]{Narcowich2006}.

The aim of the present article is to pursue the theoretical
advantages of the needlet kernel in kernel methods for spherical
nonparametric regression problems. If the kernel ridge regression
(KRR) associated with the needlet kernel is employed, the model
selection  then boils down to determining the frequency  and
regularization parameter. Due to the excellent localization in the
frequency domain, we find that the regularization parameter of KRR
can decrease arbitrarily fast for a suitable  frequency. An extreme
case is that the regularization term is not necessary for KRR in the
sense of rate optimality. This attribution is totally different from
other kernels without good localization property in the frequency
domain \cite{Cucker2002}, such as the Gaussian \cite{Minh2010} and
Abel-Poisson \cite{Freeden1998} kernels. We attribute the above
property as the first feature of the needlet kernel. Besides the
good generalization capability, some real world applications also
require the estimate to possess   the smoothness, low computational
complexity and sparsity \cite{Scholkopf2001}. This guides us to
consider the $l_{q}$ ($0< q\leq 2$) kernel regularization (KRS)
schemes associated with the needlet kernel, including  the kernel
bridge regression and kernel lasso estimate \cite{Wu2008}. The first
feature of the needlet kernel implies that the generalization
capability of all  $l_{q}$-KRS  with $0<q\leq 2$ are almost the
same, provided the regularization parameter is set to be small
enough. However, such a setting makes there be no difference among
all $l_{q}$-KRS with  $0<q\leq 2$, as each of them behaves similar
as the least squares. To distinguish different behaviors of the
$l_{q}$-KRS, we should establish a similar result for a large
regularization parameter. By the aid of a probabilistic cubature
formula and the the excellent localization property in both
frequency and spacial domain of the needlet kernel, we find that all
 $l^q$-KRS with $0<q\leq 2$   can attain almost the same almost optimal
generalization error bounds, provided the regularization parameter
is not larger than $\mathcal O(m^{q-1}\varepsilon).$ Here $m$ is the
number of samples and $\varepsilon$ is the prediction accuracy. This
implies that the choice of $q$ does not have a strong impact in
terms of the generalization capability for $l^q$-KRS, with
relatively large regularization parameters depending on $q$. From
this perspective, $q$ can be specified by other no generalization
criteria like smoothness, computational complexity and sparsity. We
consider it as the other feature of the needlet kernel.
%
% Although it
%doesn't require any additional computation in implementation, the
%aforementioned  truncation operator sends elements in the original
%reproducing kernel Hilbert space out, which makes the statistical
%behavior of the truncated kernel ridge regression be unclear. That
%is,  one aims to build a estimator in the reproducing kernel Hilbert
%space associated with the needlet kernel, but
%  actually obtains  an estimator out of it. According to needlet kernel's perfect
%localization in the spacial domain, we  succeed in removing the
%truncation operator in the cost of presenting a strict restriction
%to the regularization parameter. That is, we can derive the same
%generalization error bound for the kernel ridge regression
%associated with the needlet kernel as its truncated version. we
%find, to the best of our knowledge, this property is not holds for
%other kernels without good localization property in the spacial
%domain, such as the polynomial  \cite{Minh2006} and spherical
%harmonic \cite{Cao2013} kernels. We attribute the above property as
%the second advantage of the needlet kernel.
%

The reminder of the paper is organized as follows. In the next
section, the needlet kernel together with its important properties
such as the reproducing property, compressible property and best
approximation property is introduced. In Section 3, we study the
generalization capability of the kernel ridge regression associated
with the needlet kernel. In Section 4, we consider the
generalization capability of the $l^q$ kernel regularization
schemes, including the kernel bridge regression and kernel lasso. In
Section 5,  we provide the proofs  of the main results. We conclude
the paper with some useful remarks in the last section.

\section{The needlet kernel}

Let $\mathbf S^d$ be the unit sphere embedded into $\mathbf
R^{d+1}$.
 For integer $k\geq0$, the restriction
to $ \mathbf{S}^{d}$ of a homogeneous harmonic polynomial of degree
$k$ on the unit sphere is called a spherical harmonic of degree $k$.
The class of all spherical harmonics of degree $k$ is denoted by
$\mathbf{H}^{d}_k$, and the class of all spherical harmonics of
degree $k\leq n$ is denoted by $\Pi_n^{d}$. Of course,
$\Pi_n^{d}=\bigoplus_{k=0}^n\mathbf{H}^{d}_k$, and it comprises the
restriction to $ \mathbf{S}^{d}$  of all algebraic polynomials in
$d+1$ variables of total degree not exceeding $n$. The dimension of
$\mathbf{H}^{d}_k$ is given by
$$
                      D_k^{d}:=\mbox{dim}\ \mathbf{H}^{d}_k=\left\{\begin{array}{ll}
                     \frac{2k+d-1}{k+d-1}{{k+d-1}\choose{k}}, & k\geq 1;\\
                      1, & k=0,
                     \end{array}
                       \right.
$$
and that of $\Pi_n^{d}$ is $\sum_{k=0}^nD^{d}_k=D_n^{d+1}\sim n^d$.

The addition formula establishes a connection between spherical
harnomics of degree $k$ and the Legendre polynomial $P_k^{d+1}$
\cite{Freeden1998}:
\begin{equation}\label{jiafadingli}
                   \sum_{l=1}^{D_k^{d}}Y_{k,l}(x)Y_{k,l}(x')
                   =\frac{D_k^{d}}{|\mathbf S^d|}P_k^{d+1}(x\cdot x'),
\end{equation}
where  $P_k^{d+1}$ is the Legendre polynomial with degree $k$ and
dimension $d+1$. The Legendre polynomial $P_k^{d+1}$ can be
normalized such that $P_k^{d+1}(1)=1,$ and satisfies the
orthogonality relations
$$
                      \int_{-1}^1P_k^{d+1}(t)P^{d+1}_j(t)(1-t^2)^{\frac{d-2}2}dt
                      =\frac{|\mathbf S^d|}{|\mathbf S^{d-1}|D_k^{d}}\delta_{k,j},
$$
where $\delta_{k,j}$ is the usual Kronecker symbol.

The following Funk-Hecke formula establishes a connection between
spherical harmonics and  function $\phi\in L^1([-1,1])$
\cite{Freeden1998}
\begin{equation}\label{funk-heck}
                  \int_{\mathbf{S}^{d}}\phi(x\cdot x')H_k(x')d\omega(y)=B(\phi,k)H_k(x),
\end{equation}
where
$$
              B(\phi,k)=|\mathbf S^{d-1}|\int_{-1}^1
                P_k^{d+1}(t)\phi(t)(1-t^2)^{\frac{d-2}2} dt.
$$

 A function $\eta$ is said to be admissible \cite{Narcowich20061} if
$\eta\in C^\infty[0,\infty)$  satisfies   the following condition:
$$
              \mbox{supp}\eta\subset[0,2],\eta(t)=1\ \mbox{on}\
              [0,1],\ \mbox{and}\ 0\leq\eta(t)\leq1\ \mbox{on}\
              [1,2].
$$
The needlet kernel \cite{Narcowich2006} is then defined to be
\begin{equation}\label{Needlet kernel}
                     K_n(x\cdot x')=\sum_{k=0}^\infty\eta\left(\frac{k}{n}\right)
                     \frac{D_k^{d}}{|\mathbf S^d|}P_k^{d+1}(x\cdot x'),
\end{equation}
The needlets can be deduced from the needlet kernel and a spherical
cubature formula  \cite{Brown2005,LeGia2008,Mhaskar2000}. We refer
the readers to \cite{Baldi2008,Kerkyacharian2012,Narcowich2006} for
a detailed description of the needlets. According to the definition
of the admissible function, it is easy to see that $K_n$ possess
excellent localization property in the frequency domain.   The
following Lemma \ref{Localization} that can be found in
\cite{Narcowich2006} and \cite{Brown2005} yields that  $K_n$ also
possesses perfect spacial localization property.
\begin{lemma}\label{Localization}
 Let $\eta$ be admissible. Then for every $k>0$ and
$r\geq0$ there exists a constant $C$ depending only on $k,r,d$ and
$\eta$ such that
$$
              \left|\frac{d^r}{dt^r}K_n(\cos\theta)\right|
              \leq C\frac{n^{d+2r}}{(1+n\theta)^k},\ \theta\in[0,\pi].
$$
\end{lemma}

For   $f\in L^1(\mathbf S^d)$, we write
$$
           K_n*f(\xi):=\int_{\mathbf
           S^d}K_n(x\cdot x')f(x')d\omega(x').
$$
We also denote by $E_N(f)_p$   the best approximation error of $f\in
L^p(\mathbf S^d)$ ($p\geq 1$) from $\Pi_N^{d}$, i.e.
$$
              E_N(f)_p:=\inf_{P\in\Pi_N^{d}}\|f-P\|_{L^p(\mathbf
              S^d)}.
$$
Then the needlet kernel $K_n$ satisfies the following Lemma
\ref{PROPERTY}, which can be deduced from \cite{Narcowich2006}.

\begin{lemma}\label{PROPERTY}
  $K_n$ is a reproducing kernel for $\Pi_n^{d}$, that is
$K_n*P=P$ for $P\in\Pi_n^{d}$. Moreover, for any $f\in L^p(\mathbf
S^d),1\leq p\leq\infty$, we have $K_n*f\in\Pi_{2n}^{d}$, and
$$
              \|K_n*f\|_{L^p(\mathbf S^d)}\leq C\|f\|_{L^p(\mathbf
              S^d)},\ \mbox{and}\ \|f-K_n*f\|_{L^p(\mathbf S^d)}\leq
              CE_n(f)_p,
$$
where $C$ is a constant depending only on $d,p$ and $\eta$.
\end{lemma}

It is obvious that $K_n$ is a semi-positive definite kernel, thus it
follows from the known Mercer theorem \cite{Minh2006} that $K_n$
corresponds a reproducing kernel Hilbert space (RKHS), $\mathcal
H_K$.

\begin{lemma}\label{RKHS}
  Let $K_n$ be defined above, then the reproducing
kernel Hilbert space associated with $K_n$ is the space $
                \Pi_{2n}^{d}
$ with the inner product:
$$
         \left\langle f,g\right\rangle_{K_n}:=\sum_{k=0}^\infty
         \sum_{j=1}^{D_j^{d}}\eta(k/n)^{-1}\hat{f}_{k,j}\hat{g}_{k,j},
$$
where $\hat{f}_{k,j}=\int_{\mathbf S^d}f(x)Y_{k,j}(x)d\omega(x)$.
\end{lemma}

\section{Kernel ridge regression associated with the needlet kernel}

In spherical nonparametric regression problems with   predictor
variables $X\in\mathcal X=\mathbf S^d$ and   response variables
$Y\in\mathcal Y\subseteq\mathbf R$, we observe $m$ i.i.d. samples
${\bf z}_m=(x_i,y_i)_{i=1}^m$ from an unknown distribution $\rho$.
Without loss of generality, it is always assumed that $\mathcal
Y\subseteq[-M,M]$ almost surely, where $M$ is a positive constant.
One natural measurement of the estimate $f$ is the generalization
error,
$$
                     \mathcal E(f):=\int_Z(f(X)-Y)^2d\rho,
$$
which is minimized by the regression function \cite{Gyorfi2002}
defined by
$$
                     f_\rho(x):=\int_{\mathcal Y}Yd\rho(Y|x).
$$
Let $L^2_{\rho_{_X}}$ be the Hilbert space of $\rho_X$ square
integrable functions, with norm  $\|\cdot\|_\rho.$ In the setting of
$f_\rho\in L^2_{\rho_{_X}}$, it is well known that, for every $f\in
L^2_{\rho_X}$, there holds
\begin{equation}\label{equality}
                     \mathcal E(f)-\mathcal E(f_\rho)=\|f-f_\rho\|^2_\rho.
\end{equation}

We  formulate the learning problem in terms of  probability rather
than expectation. To this end, we present a formal way to measure
the performance of learning schemes in   probability. Let
$\Theta\subset L_{\rho_X}^2$ and $\mathcal M(\Theta)$ be the class
of all Borel measures $\rho$  such that $f_\rho\in\Theta$. For each
$\varepsilon>0$, we enter into a competition over all estimators
based on $m$ samples $\Phi_m: {\bf z}\mapsto f_{\bf z}$ by
$$
           {\bf AC}_m(\Theta,\varepsilon):=\inf_{f_{\bf
           z}\in\Phi_m}\sup_{\rho\in \mathcal M(\Theta)}\mathbf P^m\{{\bf
           z}:\|f_\rho-f_{\bf z}\|_\rho^2>\varepsilon\}.
$$
%and
%$$
%          e_m(\Theta):=\inf_{f_{\bf
%           z}\in\Phi_m}\sup_{\rho\in \mathcal
%          M(\Theta)}\mathbf E^m\left\{\|f_\rho-f_{{\bf z}}\|^2_\rho\right\}.
%$$

As it is impossible to obtain a nontrivial convergence rate wtihout
imposing any restriction on the distribution $\rho$
\cite[Chap.3]{Gyorfi2002}, we should introduce certain prior
information.  Let $\mu\geq 0$. Denote the Bessel-potential Sobolev
class
 $W_r$ \cite{Mhaskar2010} to be all   $f$ such that
$$
      \|f\|_{W_r}:=\left\|\sum_{k=0}^\infty(k+(d-1)/2)^r
      P_lf\right\|_2\leq 1,
$$
where
$$
           P_lf=\sum_{j=1}^{D_{k}^d}\left\langle
           f,Y_{k,j}\right\rangle Y_{k,j}.
$$
It follows from the well known Sobolev embedding theorem that
$W_r\subset C(\mathbf S^d)$, provided $r>d/2$. In our analysis, we
assume $f_\rho\in W_r$.

The learning scheme employed in this section is the following kernel
ridge regression (KRR) associated with the needlet kernel
\begin{equation}\label{RLS for Kn}
                    f_{{ \bf z},\lambda}:=\arg\min_{f\in \mathcal H_K}
                    \left\{\frac{1}{m}\sum_{i=1}^{m}(f(x_{i})-y_{i})^2
                    +\lambda\|f\|^2_{K_n}\right\}.
\end{equation}
Since $y\in[-M,M],$ it is easy to see that $\mathcal E(\pi_Mf)\leq
\mathcal E(f)$  for arbitrary $f\in L_{\rho_X}^2$, where $\pi_M
u:=\min\{M,|u|\}sgn(u)$ is the truncation operator. As there isn't
any additional computation for employing the  truncation operator,
the truncation operator has been used in large amount of papers, to
just name a few,
\cite{Cao2013,Devore2006,Gyorfi2002,Lin2014,Minh2006,Wu2008,Zhou2006}.
  The
 following Theorem \ref{THEOREM 1} illustrates the generalization
 capability of KRR
 associated with the needlet kernel and reveals
  the first feature of
 the needlet kernel.

\begin{theorem}\label{THEOREM 1}
 Let $f_\rho\in W_r$ with $r>d/2$,  $m\in\mathbf N$,
$\varepsilon>0$ be any real number, and $n\sim\varepsilon^{-r/d}$.
If $f_{{\bf z},\lambda}$ is defined as in (\ref{RLS for Kn}) with
$0\leq \lambda\leq M^{-2}\varepsilon$, then there exist positive
constants $C_i,$ $i=1,\dots,4,$ depending only on $M$, $\rho$, and
$d$,  $\varepsilon_0>0$ and $\varepsilon_-,\varepsilon_+$ satisfying
\begin{equation}\label{phase1}
            C_1m^{-2r/(2r+d)}\leq\varepsilon_-\leq\varepsilon_+\leq
            C_2(m/\log m)^{-2r/(2r+d)},
\end{equation}
such that for any $
         \varepsilon<\varepsilon_-$,
\begin{equation}\label{negative1}
         \sup_{f_\rho\in W_r}\mathbf P^m\{{\bf
                  z}:\|f_\rho-\pi_Mf_{{\bf z},\lambda}\|_\rho^2>\varepsilon\}
                  \geq{\bf AC}_m(W_r,\varepsilon)\geq \varepsilon_0,
\end{equation}
and for any $
                  \varepsilon\geq\varepsilon_+$,
\begin{equation}\label{Theorem 1}
               e^{-C_3m\varepsilon}
                \leq
                {\bf
               AC}_m(W_r,\varepsilon)
                \leq
               \sup_{f_\rho\in W_r}\mathbf P^m\{{\bf
                  z}:\|f_\rho-\pi_Mf_{{\bf z},\lambda}\|_\rho^2>\varepsilon\}
                  \leq e^{-C_4m\varepsilon}.
\end{equation}
\end{theorem}

We give several remarks on Theorem \ref{THEOREM 1} below.  In some
real world applications, there are only $m$ data available, and the
purpose of learning is to produce an   estimate with the prediction
error at most $\varepsilon$ and statisticians  are required to
assess  the probability of success. It is obvious that the
probability depends heavily on $m$ and $\varepsilon$. If $m$ is too
small, then there isn't any   estimate that can finish the learning
task with small $\varepsilon$.
 This fact is quantitatively verified by the inequality
 (\ref{negative1}).
 More specifically, (\ref{negative1}) shows
that if the learning task is to  yield an accuracy at most
$\varepsilon\leq \varepsilon_-$, and other than the prior knowledge,
$f_\rho\in W_r$, there are only $m\leq \varepsilon_-^{-(2r+d)/(2r)}$
data available, then all learning schemes, including  KRR associated
with the needlet kernel, may fail with high probability. To
circumvent it, the only way is to acquire more samples, just as
inequalities (\ref{Theorem 1}) purport to show. (\ref{Theorem 1})
says that if the number of samples achieves
$\varepsilon_+^{-(2r+d)/(2r)}$, then the probability of success of
 KRR is at least $1-e^{-{C_4m\varepsilon}}$. The first inequality
(lower bound) of (\ref{Theorem 1})   implies that this confidence
can not be improved further. The values of $\varepsilon_-$ and
$\varepsilon_+$ thus  are very critical since   the  smallest number
of samples to finish the learning task lies in the interval
$[\varepsilon_-,\varepsilon_+]$.  Inequalities (\ref{phase1})
depicts that, for KRR, there holds
$$
         [\varepsilon_-,\varepsilon_+]\subset[C_1m^{-2r/(2r+d)},C_2(
           m/\log m)^{-2r/(2r+d)}].
$$
 This implies   that the interval
$[\varepsilon_-,\varepsilon_+]$  is almost the shortest one in the
sense that up to a logarithmic factor, the upper bound and lower
bound of the interval are asymptotically identical. Furthermore,
Theorem \ref{THEOREM 1} also presents a sharp phase transition
phenomenon of  KRR. The behavior of the   confidence function
changes dramatically within the critical interval
$[\varepsilon_-,\varepsilon_+]$. It drops from a constant
$\varepsilon_0$ to an exponentially small quantity.   All the above
assertions show that the learning performance of  KRR is essentially
revealed in Theorem \ref{THEOREM 1}.

An interesting finding in Theorem \ref{THEOREM 1} is that  the
regularization parameter of  KRR can decrease arbitrarily fast,
provided it is smaller than $M^{-2}\varepsilon$. The  extreme case
is that the   least-squares
 possess  the same generalization performance as  KRR. It is not surprised
 in the realm of
nonparametric regression, due to the needlet kernel's localization
property in the frequency domain. Via controlling the frequency of
the needlet kernel, $\mathcal H_K$ is essentially a linear space
with finite dimension. Thus,   \cite[Th.3.2\& Th.11.3]{Gyorfi2002}
together with Lemma \ref{JACKSON} in the present paper automatically
yields the optimal learning rate of the  least squares associated
with the needlet kernel in the sense of expectation. Differently,
Theorem \ref{THEOREM 1}  presents an exponential confidence estimate
for  KRR, which together with (\ref{phase1}) makes
\cite[Th.11.3]{Gyorfi2002} be a corollary of Theorem \ref{THEOREM
1}. Theorem \ref{THEOREM 1} also shows that the purpose of
introducing regularization term in KRR is  only to conquer the
singular problem of the kernel matrix, $A:=(K_n(x_i\cdot
x_j))_{i,j=1}^m$, since  $m> D_n^{d+1}$ in our setting. Under this
circumstance,  a small $\lambda$ leads to the ill-condition of the
matrix $A+m\lambda I$ and a large $\lambda$ conducts large
approximation error.  Theorem \ref{THEOREM 1} illustrates that if
the needlet kernel is employed, then we can set
$\lambda=M^{-2}\varepsilon$ to guarantee both the small condition
number of the kernel matrix and almost generalization error bound.
From (\ref{phase1}), it is easy to deduce that to attain the optimal
learning rate $m^{-2r/(2r+d)}$, the minimal eigenvalue of the matrix
$A+m\lambda I$ is $m^{d/(2r+d)}$, which can guarantee that the
matrix inverse technique is suitable to solve (\ref{RLS for Kn}).

\section{$l^q$ kernel regularization  schemes associated with the needlet kernel}

In the last section, we analyze the generalization capability of KRR
associated with the needlet kernel. This section aims to study the
learning capability of the $l^q$ kernel regularization scheme  (KRS)
whose hypothesis space is the sample dependent hypothesis space
\cite{Wu2008} associated with $K_{n}(\cdot,\cdot)$ ,
$$
             \mathcal H_{K,{\bf
             z}}:=\left\{\sum_{i=1}^ma_iK_{n}(x_i,\cdot):a_i\in\mathbf
             R\right\}
$$
The corresponding $l^q$-KRS is defined by
\begin{equation}\label{algorihtm1}
           f_{{\bf z},\lambda,q}\in\arg\min_{f\in\mathcal H_{K,{\bf
           z}}}\left\{\frac1m\sum_{i=1}^m(f(x_i)-y_i)^2+\lambda\Omega_{\bf z}^q(f)\right\},
\end{equation}
where
$$
             \Omega_{\bf z}^q(f):=\inf_{(a_1,\dots,a_n)\in\mathbf R^n}
             \sum_{i=1}^m|a_i|^q, \mbox{for}\
            f=\sum_{i=1}^ma_iK_{n}(x_i,\cdot).
$$

With different choices of the order $q$, (\ref{algorihtm1}) leads to
various specific forms of the $l_{q}$ regularizer.  $f_{{\bf
z},\lambda,2}$ corresponds to the kernel ridge regression
\cite{Scholkopf2001}, which smoothly shrinks the coefficients toward
zero and $f_{{\bf z},\lambda,1}$ leads to the LASSO
\cite{Tibshirani1995}, which sets small coefficients exactly at zero
and thereby also serves as a variable selection operator.   The
varying forms and properties of $f_{{\bf z},\lambda,q}$ make the
choice of order $q$ crucial in applications. Apparently, an optimal
$q$ may depend on many factors such as the learning algorithms, the
purposes of studies and so forth. The following Theorem \ref{THEOREM
2} shows that if the needlet kernel is utilized in   $l^q$-KRS, then
$q$ may not have an important impact in the generalization
capability for a large range of regularization parameters in the
sense of rate optimality.

Before setting the main results, we should at first introduce a
restriction to the marginal distribution $\rho_X$. Let $J$ be the
identity mapping
$$
            L^2_{\rho_X}   ~~ {\stackrel{J}{\longrightarrow}}~~ L^2(\mathbf B^d) .
$$
and $D_{\rho _{X}}=$ $\Vert J\Vert .$ $D_{\rho _{X}}$ is called the
distortion of $\rho _{X}$ (with respect to the Lebesgue measure) \cite{Zhou2006}%
, which measures how much $\rho _{X}$ distorts the Lebesgue measure.

\begin{theorem}\label{THEOREM 2}
 Let $f_\rho\in W_r$ with $r>d/2$, $D_{\rho _{X}}<\infty$,  $m\in\mathbf N$,
$\varepsilon>0$ be any real number, and $n\sim\varepsilon^{-r/d}$.
If $f_{{\bf z},\lambda,q}$ is defined as in (\ref{algorihtm1}) with
$\lambda\leq m^{1-q}\varepsilon$ and $0<q\leq 2$, then there exist
positive constants $C_i,$ $i=1,\dots,4,$ depending only on $M$,
$\rho$, $q$ and $d$,  $\varepsilon_0>0$ and
$\varepsilon_m^-,\varepsilon_m^+$ satisfying
\begin{equation}\label{phase2}
            C_1m^{-2r/(2r+d)}\leq\varepsilon_m^-\leq\varepsilon_m^+\leq
            C_2(m/\log m)^{-2r/(2r+d)},
\end{equation}
such that for any $
         \varepsilon<\varepsilon_m^-$,
\begin{equation}\label{negative2}
         \sup_{f_\rho\in W_r}\mathbf P^m\{{\bf
                  z}:\|f_\rho-\pi_Mf_{{\bf z},\lambda,q}\|_\rho^2>\varepsilon\}
                  \geq{\bf AC}_m(W_r,\varepsilon)\geq \varepsilon_0,
\end{equation}
and for any $
                  \varepsilon\geq\varepsilon_m^+$,
\begin{equation}\label{Theorem 2}
               e^{-C_3m\varepsilon}
               \leq
                {\bf
               AC}_m(W_r,\varepsilon)
                \leq
               \sup_{f_\rho\in W_r}\mathbf P^m\{{\bf
                  z}:\|f_\rho-\pi_Mf_{{\bf z},\lambda,q}\|_\rho^2>\varepsilon\}
                  \leq e^{-C_4D^{-1}_{\rho_X}m\varepsilon}.
\end{equation}
\end{theorem}

Compared with KRR (\ref{RLS for Kn}), a common consensus is that
$l^q$-KRS (\ref{algorihtm1}) may bring a certain additional interest
such as the  sparsity for suitable choice of $q$. However, it should
be noticed that this assertion may not always be true. This
conclusion depends heavily on the value of the regularization
parameter. If the the regularization parameter is extremely small,
then   $l^q$-KRS for any $q\in(0,2]$ behave similar as the least
squares. Under this circumstance, Theorem \ref{THEOREM 2} obviously
holds due to the conclusion of Theorem \ref{THEOREM 1}. To
distinguish the character of  $l^q$-KRS with different $q$, one
should consider a relatively large regularization parameter. Theorem
\ref{THEOREM 2} shows that for a large range of regularization
parameters, all the $l^q$-KRS associated with the needlet kernel can
attain the same, almost optimal, generalization error bound. It
should be highlighted that the quantity $m^{q-1}\varepsilon$ is, to
the best of knowledge, almost the largest value of the
regularization parameter among all the existing results. We
encourage the readers to compare our result with the results in
\cite{Lin2014,Shi2011,Tong2010,Wu2008}. Furthermore, we find that
$m^{q-1}\varepsilon$ is sufficient to embody the feature of $l^q$
kernel regularization schemes. Taking the kernel lasso for example,
the regularization parameter derived in Theorem \ref{THEOREM 2}
asymptotically equals to $\varepsilon$. It is to see that, to yield
a prediction accuracy $\varepsilon,$ we have
$$
 f_{{\bf z},\lambda,1}\in\arg\min_{f\in\mathcal H_{K,{\bf
           z}}}\left\{\frac1m\sum_{i=1}^m(f(x_i)-y_i)^2+\lambda\Omega_{\bf z}^1(f)\right\},
$$
and
$$
          \frac1m\sum_{i=1}^m(f(x_i)-y_i)^2\leq \varepsilon.
$$
According to the structural risk minimization principle and
$\lambda=\varepsilon$, we obtain
$$
        \Omega_{\bf z}^1(f_{{\bf z},\lambda,1})\leq C.
$$

Intuitively,  the generalization capability of   $l^q$-KRS
(\ref{algorihtm1}) with a large regularization parameter may depend
on the choice of $q$.
 While from Theorem \ref{THEOREM 2}   it follows  that the learning schemes
defined by (\ref{algorihtm1}) can indeed achieve the same
asymptotically optimal rates for all   $q\in(0,\infty)$. In other
words, on the premise of embodying the feature of  $l^q$-KRS with
different $q$, the choice of $q$ has no influence on the
generalization capability in the sense of rate optimality. Thus, we
can determine
  $q$  by taking other
non-generalization considerations such as the smoothness, sparsity,
and computational complexity into account. Finally, we explain the
reason for this phenomenon by taking needlet kernel's perfect
localization property in the spacial domain into account. To
approximate $f_\rho(x)$, due to the localization property of $K_n$,
we can construct an approximant in $\mathcal H_{{\bf z},K}$ with a
few $K_n(x_i,x)$'s whose centers $x_i$ are near to $x$. As $f_\rho$
is bounded by $M$, then the coefficient of these terms are also
bounded. That is, we can construct, in $\mathcal H_{{\bf z},K}$, a
good approximant, whose $l^q$ norm is bounded for arbitrary $0<
q<\infty$. Then, using the standard error decomposition technique in
\cite{Cucker2001} that divide the generalization error into the
approximation error and sample error, the approximation error of
$l^q$-KRS is independent of $q$. For the sample error, we can tune
$\lambda$ that may depend on $q$ to offset the effect of $q$. Then,
a generalization error estimate independent of $q$ is natural.

\section{Proofs}
In this section, we present the proof of Theorem \ref{THEOREM 1} and
Theorem \ref{THEOREM 2}, respectively.
\subsection{Proof of Theorem \ref{THEOREM 1}}

For the sake of brevity, we set $f_n=K_n*f_\rho$. Let
$$
                      \mathcal S(\lambda,m,n):=\left\{\mathcal{E}(\pi_M f_{{\bf z},\lambda} )
                      -\mathcal{E}_{\bf
                      z}(\pi_M f_{{\bf z},\lambda})  +\mathcal{E}_{\bf
                      z}(f_n)-\mathcal{E}(f_n)\right\}.
$$
Then it is easy to deduce that
\begin{equation}\label{wcfj}
                    \mathcal{E}(\pi_M f_{{\bf z},\lambda} )-\mathcal{E}(f_{\rho})\leq
                    \mathcal S(\lambda,m,n)+\mathcal D_n(\lambda),
\end{equation}
where $
             \mathcal D_n(\lambda):=\|f_n-f_\rho\|_{\rho}^2+\lambda\|f_n\|_{K_n}^2.
$ If we set $
            \xi_1:=(\pi_M(f_{{\bf z},\lambda})(x)-y)^2-(f_\rho(x)-y)^2,
$ and $
            \xi_2:=(f_n(x)-y)^2-(f_\rho(x)-y)^2,
$ then
$$
            \mathbf E(\xi_1)=
            \int_Z\xi_1(x,y)d\rho=\mathcal E(\pi_M(f_{{\bf z},\lambda})(x))-\mathcal
            E(f_\rho),\ \mbox{and}\
            \mathbf E(\xi_2)=\mathcal E(f_n)-\mathcal E(f_\rho).
$$
Therefore, we can rewrite the sample error as
\begin{equation}\label{yangbenwucha1}
            S(\lambda,m,n)=\left\{\mathbf E(\xi_1)-\frac1m\sum_{i=1}^m\xi_1(z_i)\right\}
            +\left\{\frac1m\sum_{i=1}^m\xi_2(z_i)-\mathbf E(\xi_2)\right\}=:
            \mathcal S_1+\mathcal S_2.
\end{equation}

The aim of this subsection is to bound $\mathcal D_n(\lambda)$,
$\mathcal S_1$ and $\mathcal S_2$, respectively. To bound $\mathcal
D_n(\lambda)$, we need the following two lemmas. The first one is
the Jackson-type inequality that can be deduced from
\cite{Mhaskar2010,Narcowich2006} and the second one describes the
RKHS norm of $f_n$.

\begin{lemma}\label{JACKSON}
 Let $f\in W_r$. Then there exists a constant depending only on $d$
 and $r$ such that
$$
                     \|f-f_n\|\leq Cn^{-2r},
$$
where $\|\cdot\|$ denotes the uniform norm on the sphere.
\end{lemma}

\begin{lemma}\label{BOUND RKHS}
 Let $f_n$ be defined as above. Then we have
$$
           \|f_n\|^2_{K_{n}}\leq M^2.
$$
\end{lemma}

\begin{proof}
Due to the addition formula (\ref{jiafadingli}), we have
$$
           K_n(x\cdot y)=\sum_{k=0}^{n}\eta\left(\frac{k}{n}\right)
           \left\{\sum_{j=1}^{D_j^d}Y_{k,j}(x)Y_{k,j}(y)\right\}
           =\sum_{k=0}^{n}\eta\left(\frac{k}{n}\right)
           \frac{D_k^d}{\Omega_d}P^{d+1}_k(x\cdot y).
$$
Since
$$
           K_n*f(x)=\int_{\mathbf S^d}K_n(x\cdot y)f(y)d\omega(y),
$$
it follows from the Funk-Hecke formula (\ref{funk-heck}) that
\begin{eqnarray*}
           \widehat{K_n*f}_{u,v}
           &=&
           \int_{\mathbf S^d}K_n*f(x)Y_{u,v}(x)d\omega(x)
           =\int_{\mathbf S^d}\int_{\mathbf S^d}K_n(x\cdot x')f(x')
           d\omega(x')Y_{u,v}(x)d\omega(x)\\
           &=&
           \int_{\mathbf S^d}f(x')\int_{\mathbf S^d}K_n(x\cdot x')
           Y_{u,v}(x)d\omega(x)d\omega(x')\\
           &=&
           \int_{\mathbf S^d}|\mathbf S^{d-1}|\int_{-1}^1K_n(t)
           P_u^{d+1}(t)(1-t^2)^\frac{d-2}2dt
           Y_{u,v}(x')f(x')d\omega(x')\\
           &=&
           |\mathbf S^{d-1}|\hat{f}_{u,v}\int_{-1}^1K_n(t)P_u^{d+1}(t)(1-t^2)^\frac{d-2}2dt.
\end{eqnarray*}
Moreover,
\begin{eqnarray*}
           \int_{-1}^1K_n(t)P_u^{d+1}(t)(1-t^2)^\frac{d-2}2dt
           &=&
           \int_{-1}^1\sum_{k=0}^{n}\eta\left(\frac{u}{n}\right)
           \frac{D_k^d}{|\mathbf S^{d}|}P^{d+1}_u(t)P_u^{d+1}(t)(1-t^2)^\frac{d-2}2dt\\
           &=&
           \int_{-1}^1\eta\left(\frac{u}{n}\right)\frac{D_u^d}{|\mathbf S^{d}|}
           P^{d+1}_u(t)P_u^{d+1}(t)(1-t^2)^\frac{d-2}2dt\\
           &=&
           \eta\left(\frac{u}{n}\right)\frac{D_u^d}{|\mathbf S^{d}|}
           \frac{|\mathbf S^{d}|}{|\mathbf S^{d-1}|D_u^d}
           =
           \eta\left(\frac{u}{n}\right)\frac{1}{|\mathbf S^{d-1}|}.
\end{eqnarray*}
Therefore,
$$
           \widehat{K_n*f}_{u,v}
           =\eta\left(\frac{u}{n}\right)\hat{f}_{u,v}.
$$
This implies
\begin{eqnarray*}
           \|K_n*f\|^2_{K_{n}}
           &=&
           \sum_{u=0}^{n}\eta\left(\frac{u}{n}\right)^{-1}
           \sum_{v=1}^{D_u^d}(\widehat{K_n*f}_{u,v})^2\\
           &\leq&
           \sum_{u=0}^{n}\sum_{v=1}^{D_u^d}\hat{f}_{u,v}^2
           \leq
           \|f\|^2_{L^2(\mathbf S^d)}\leq M^2.
\end{eqnarray*}
The proof of Lemma \ref{BOUND RKHS} is completed.
\end{proof}

Based on the above two lemmas, it is easy to deduce an upper bound
of $\mathcal D_n(\lambda)$.

\begin{proposition}\label{BOUND Dn}
Let $f\in W_r$. There exists a positive constant $C$ depending only
on $r$ and $d$ such that
$$
             \mathcal D_n(\lambda)\leq Cn^{-2r}+M^2\lambda
$$

\end{proposition}

In the rest of this subsection, we will bound $\mathcal S_1$ and
$\mathcal S_2$ respectively. The approach used here is somewhat
standard in learning theory. $\mathcal S_2$ is a typical quantity
that can be estimated by probability inequalities.
 We shall bound it by the following one-side Bernstein inequality   \cite{Cucker2001}.

\begin{lemma}\label{BERNSTEIN INEQUALITY}
 Let $\xi$ be a random variable on a probability space $Z$ with
mean $\mathbf E(\xi)$, variance $\sigma^2(\xi)=\sigma^2$. If
$|\xi(z)-\mathbf E(\xi)|\leq M_\xi$ for almost all ${\bf z}\in Z$.
then, for all $\varepsilon>0$,
$$
           \mathbf P^m\left\{\frac1m\sum_{i=1}^m\xi(z_i)-
           \mathbf E(\xi)\geq\varepsilon\right\}
           \leq \exp
           \left\{-\frac{m\varepsilon^2}
           {2\left(\sigma^2+\frac13M_\xi\varepsilon\right)}\right\}.
$$
\end{lemma}

By the help of the above lemma, we can deduce the following bound of
$\mathcal S_2$.

\begin{proposition}\label{BOUND S2}
 For every $0<\delta<1$, with confidence at least
$$
         1-\exp\left(-\frac{3m\varepsilon^2}{48M^2\left(2\|f_n-f_\rho\|^2_{\rho}
            +\varepsilon\right)}\right)
$$
 there holds
$$
           \frac1m\sum_{i=1}^m\xi_2(z_i)-\mathbf E(\xi_2)\leq
           \varepsilon.
$$
\end{proposition}

\begin{proof}
It follows from Lemma \ref{PROPERTY} that $\|f_n\|_\infty\leq M$,
which together with $|f_\rho(x)|\leq M$ yields that
$$
            |\xi_2|\leq(\|f_n\|_\infty+M)(\|f_n\|_\infty+M)\leq4M^2.
$$
Hence
 $|\xi_2-E(\xi_2)|\leq8M^2$. Moreover, we have
$$
            \mathbf E(\xi_2^2)=
            \mathbf E((f_n(X)-f_\rho(X)^2\times(f_n(X)-Y)+(f_\rho(X)-Y))^2)
            \leq16M^2\|f_n-f_\rho\|^2_\rho,
$$
which implies that
$$
            \sigma^2(\xi_2)\leq \mathbf E(\xi_2^2)\leq16M^2\|f_n-f_\rho\|^2_\rho.
$$
Now we apply Lemma \ref{BERNSTEIN INEQUALITY} to $\xi_2$. It asserts
that for any $t>0$,
$$
            \frac1m\sum_{i=1}^m\xi_2(z_i)-\mathbf E(\xi_2)\leq t
$$
with confidence at least
$$
            1-\mbox{exp}\left(-\frac{mt^2}{2\left(\sigma^2(\xi_2)+\frac83M^2t\right)}\right)
            \geq
            1-\mbox{exp}\left(-\frac{3mt^2}{48M^2\left(2\|f_n-f_\rho\|^2_\rho
            +t\right)}\right).
$$
This implies the desired estimate.
\end{proof}

It is more difficult to estimate $\mathcal S_1$ because $\xi_1$
involves the sample ${\bf z}$ through $f_{{\bf z},\lambda}$. We will
use the idea of empirical risk minimization to bound this term by
means of  covering number  \cite{Cucker2001}.  The main tools are
the following three lemmas.

\begin{lemma}\label{DIM RELATION}
Let $V_{k}$ be a $k$-dimensional function space defined on $\mathbf
S^d$. Denote by $\pi_MV_k=\{\pi_Mf:f\in V_k\}$. Then
$$
            \log \mathcal{N}(\pi_MV_{k}, \eta) \leq c k \log \frac{M}\eta,
$$
where $c$ is a positive constant and $\mathcal N(\pi_MV_{k}, \eta)$
is the  covering number associated with the uniform norm that
denotes the number of elements in least $\eta$-net of $\pi_MV_k$.
\end{lemma}

Lemma \ref{DIM RELATION} is a direct result through combining
  \cite[Property 1]{Maiorov1999} and
\cite[P.437]{Maiorov2006}. It shows that the covering number of a
bounded functional space can be also bounded properly. The following
ratio probability inequality is a standard result in learning theory
\cite{Cucker2001}. It deals with variances for a function class,
since the Bernstein inequality takes care of the variance well only
for a single random variable.

\begin{lemma}\label{CONCENTRATION INEQUALITY 1}
 Let $\mathcal G$ be a set of functions on $\mathcal Z$ such that, for
some $c\geq 0$, $|g-\mathbf E(g)|\leq B$ almost everywhere and
$\mathbf E(g^2)\leq c\mathbf E(g)$ for each $g\in\mathcal G$. Then,
for every $\varepsilon>0$,
$$
             \mathbf P^m\left\{\sup_{f\in\mathcal G}
             \frac{\mathbf E(g)-\frac1m\sum_{i=1}^mg(z_i)}{\sqrt{ \mathbf E(g)+\varepsilon}}
             \geq\sqrt{\varepsilon}\right\}\leq
             \mathcal N(\mathcal G,\varepsilon)
             \exp\left\{-\frac{m\varepsilon}{2c+\frac{2B}3}\right\}.
$$
\end{lemma}

Now  we are in a position to give an upper bound of $\mathcal S_2$.

\begin{proposition}\label{BOUND S1}
For all $\varepsilon>0$,
$$
         \mathcal S_1\leq \frac12\mathcal E(\pi_M f_{{\bf z},\lambda})-\mathcal
         E(f_\rho)+\varepsilon
$$
holds with confidence at least
$$
              1-
              \exp\left\{cn^d\log\frac{4M^2}{\varepsilon}
              -\frac{3m\varepsilon}{128M^2}\right\}.
$$
\end{proposition}

\begin{proof}
 Set
$$
             \mathcal F:=\{( f (X)-Y)^2-(f_\rho(X)-Y)^2:f\in \pi_M\mathcal
             H_K\}.
$$
Then for $g\in\mathcal F,$  there exists $f\in \mathcal H_K$ such
that $g({  Z})=(\pi_M f (X)-Y)^2-(f_\rho(X)-Y)^2$. Therefore,
$$
             \mathbf E(g)=\mathcal E(\pi_M f )-\mathcal E(f_\rho)\geq0,
             \ \ \frac1m\sum_{i=1}^mg(z_i)=\mathcal E_{\bf z}(\pi_M(f))-\mathcal E_{\bf z}
             (f_\rho).
$$
Since $|\pi_M f |\leq M$ and $|f_\rho(X)|\leq M$ almost everywhere,
we find that
$$
             |g({\bf
             z})|=|(\pi_Mf(X)-f_\rho(X))((\pi_Mf(X)-Y)+(f_\rho(X)-Y))|\leq8M^2
$$
almost everywhere. It follows that $|g({\bf z})-\mathbf
E(g)|\leq16M^2$ almost everywhere and
$$
             \mathbf E(g^2)
             \leq
             16M^2\|\pi_Mf-f_\rho\|^2_\rho=16M^2\mathbf E(g).
$$

Now we apply Lemma \ref{CONCENTRATION INEQUALITY 1} with $B=c=16M^2$
to the set of functions $\mathcal F$ and obtain that
\begin{equation}\label{RKHS important}
             \sup_{f\in \pi_M\mathcal H_K}\frac{\{\mathcal E(f)
             -\mathcal E(f_\rho)\}-\{\mathcal E_{\bf z}(f)
             -\mathcal E_{\bf z}(f_\rho)\}}{\sqrt{\{\mathcal E(f)
             -\mathcal E(f_\rho)\}+\varepsilon}}
             =
             \sup_{g\in \mathcal F}\frac{\mathbf E(g)-
             \frac1m\sum_{i=1}^mg(z_i)}{\sqrt{\mathbf E(g)+\varepsilon}}
             \leq
             \sqrt{\varepsilon}
\end{equation}
with confidence at least
$$
             1-\mathcal N(\mathcal F,\varepsilon)\mbox{exp}\left\{-\frac{3m\varepsilon}{128M^2}\right\}.
$$

Observe that for $g_1,g_2\in\mathcal F$ there exist $f_1,f_2\in
\pi_M\mathcal H_K$ such that
$$
             g_j({  Z})=(f_j(X)-Y)^2-(f_\rho(X)-Y)^2,\ j=1,2.
$$
In addition, for any $f\in \pi_M\mathcal H_K$, there holds
$$
             |g_1({  Z})-g_2({  Z})|
             =
             |(f_1(X)-Y)^2-(f_2(X)-Y)^2|
              \leq
              \leq4M\|f_1-f_2\|_\infty.
$$
 We see that for any $\varepsilon>0$, an
$\left(\frac{\varepsilon}{4M}\right)$-covering of $\pi_M\mathcal
H_K$ provides an $\varepsilon$-covering of $\mathcal F$. Therefore
$$
             \mathcal N(\mathcal F,\varepsilon)\leq
             \mathcal N\left(\pi_M\mathcal H_K,\frac{\varepsilon}{4M}\right).
$$
Then the confidence is
\begin{eqnarray*}
             1-\mathcal N(\mathcal F,\varepsilon)
             \exp\left\{-\frac{3m\varepsilon}{128M^2}\right\}
             &\geq&
             1-\mathcal N\left(\pi_M\mathcal H_K,
             \frac{\varepsilon}{4M}\right)\exp\left\{-\frac{3m\varepsilon}{128M^2}\right\}.
\end{eqnarray*}
 Since
$$
               \sqrt{\varepsilon}\sqrt{\{\mathcal E(\pi_Mf)-\mathcal E(f_\rho)\}+\varepsilon}
               \leq\frac12\{\mathcal E(\pi_Mf)-\mathcal E(f_\rho)\}+\varepsilon,
$$
it follows from (\ref{RKHS important}) and Lemma \ref{DIM RELATION}
that
$$
         \mathcal S_2\leq \frac12\mathcal E(\pi_M f_{{\bf z},\lambda})-\mathcal
         E(f_\rho)+\varepsilon
$$
holds with confidence at least
$$
              1-
              \exp\left\{cn^d\log\frac{4M^2}{\varepsilon}
              -\frac{3m\varepsilon}{128M^2}\right\}.
$$
This finishes the proof.
\end{proof}

Now we are in a position to deduce the final learning rate of the
kernel ridge regression  (\ref{RLS for Kn}). Firstly, it follows
from Propositions \ref{BOUND Dn}, \ref{BOUND S2} and \ref{BOUND S1}
that
\begin{eqnarray*}
            \mathcal
            E(\pi_Mf_{{\bf
           z},\lambda})-\mathcal E(f_\rho))
           &\leq&
            \mathcal D_n( \lambda)+\mathcal S_1+\mathcal S_2
           \leq
           C\left(n^{-2r}+\lambda M^2\right)\\
           &+&
           \frac12(\mathcal
            E(\pi_Mf_{{\bf
           z},\lambda})-\mathcal E(f_\rho))+2\varepsilon
\end{eqnarray*}
holds with confidence at least
$$
          1-\exp\left\{cn^d\log\frac{4M^2}{\varepsilon}
              -\frac{3m\varepsilon}{128M^2}\right\}
          -\exp\left(-\frac{3m\varepsilon^2}{48M^2\left(2\|f_n-f_\rho\|^2_{\rho}
            +\varepsilon\right)}\right).
$$
Then,  by setting  $\varepsilon\geq\varepsilon_+\geq C(m/\log
m)^{-2r/(2r+d)}$, $n=\left[c_0\varepsilon^{-1/(2r)}\right]$ and
$\lambda\leq M^{-2}\varepsilon$,   we get, with confidence at least
$$
        1-\exp\{-Cm\varepsilon\},
$$
there holds
$$
            \mathcal
            E(\pi_Mf_{{\bf
           z},\lambda})-\mathcal E(f_\rho)
           \leq
           4\varepsilon.
$$

The lower bound  can be more easily deduced. Actually, it can be
easily deduced from the Chapter 3 of  \cite{Devore2006}  that for
any estimator $f_{\bf z}\in \Phi_m$, there holds
$$
          \sup_{f_\rho\in W_r}\mathbf P_m\{{\bf z}:\|f_{\bf
          z}-f_\rho\|_\rho^2\geq\varepsilon\}\geq\left\{
          \begin{array}{cc}
          \varepsilon_0,  & \varepsilon< \varepsilon_-,\\
          e^{-cm\varepsilon}, & \varepsilon \geq\varepsilon_-,
          \end{array}
          \right.
$$
where $\varepsilon_0=\frac12$ and $\varepsilon_-=cm^{-2r/(2r+d)}$
for some universal constant $c$. With this,  the proof of Theorem
\ref{THEOREM 1} is completed.

\subsection{Proof of Theorem \ref{THEOREM 2}}
Before we proceed the proof, we at first present a simple
description of the methodology.  The methodology we adopted in the
proof of Theorem \ref{THEOREM 2} seems of novelty. Traditionally,
the generalization error of learning schemes in the sample dependent
hypothesis space (SDHS) is divided into the approximation,
hypothesis and sample errors (three terms) \cite{Wu2008}. All of the
aforementioned results about coefficient regularization in SDHS fall
into this style. According to \cite{Wu2008}, the hypothesis error
has been regarded as the reflection of nature of data dependence of
SDHS, and an indispensable part attributed to an essential
characteristic of learning algorithms in SDHS, compared with the
learning schemes in SIHS (sample independent hypothesis space). With
the needlet kernel $K_n$, we will  divide the generalization error
of $l^q$ kernel regularization  into the approximation and sample
errors (two terms) only.
 The core tool is needlet kernel's
excellent  localization properties in both the spacial and frequency
domain, with which the reproducing property, compressible property
and  the best approximation property can be guarantee. After
presenting
 a probabilistic cubature formula for spherical  polynomials, we can
prove that all the polynomials can be represented by via the SDHS.
This helps us to deduce the approximation error.  Since $\mathcal
H_{{\bf z},K}\subseteq\mathcal H_K$, the bound of the sample error
is as same as that in the previous subsection. Thus, We divide the
proof into three parts. The first one devotes to establish the
probabilistic cubature formula. The second one is to construct the
random approximant and study the approximation error. The third one
is to deduce the sample error and  derive the final learning rate.

To present the  probabilistic cubature formula, we need the
following two lemmas. The first one is the
 Nikolskii inequality for
spherical polynomials \cite{Mhaskar1999}.

\begin{lemma}\label{NIKOLSKII INEQUALITY}
Let $1\leq p<q\leq\infty,$ $n\geq1$ be an integer. Then
$$
          \|Q\|_{L^q(\mathbf
          S^{d})}\leq
          Cn^{\frac{d}{p}-\frac{d}{q}}\|Q\|_{L^p(\mathbf
          S^{d})},\ Q\in\Pi_s^{d}
$$
where the constant $C$ depends only on $d$.
\end{lemma}

To state the next lemma, we need introduce the following
definitions. Let $\mathcal V$ be a finite dimensional vector space
with norm $\|\cdot\|_{\mathcal V}$, and $\mathcal U\subset \mathcal
V^*$ be a finite set. Here $\mathcal V^*$ denotes the dual space of
$\mathcal V$. We say that $\mathcal U$ is a norm generating set for
$\mathcal V$ if the mapping $T_{\mathcal U}: \mathcal
V\rightarrow\mathbf R^{Card(\mathcal U)}$ defined by $T_{\mathcal
U}(x)=(u(x))_{u\in \mathcal U}$ is injective, where $Card(\mathcal
U)$ is the cardinality of the set $\mathcal U$ and $T_{\mathcal U}$
is named as the sampling operator. Let $\mathcal W:=T_{\mathcal
U}(\mathcal V)$ be the range of $T_{\mathcal U}$, then the
injectivity of $T_{\mathcal U}$ implies that $T_{\mathcal
U}^{-1}:\mathcal W\rightarrow \mathcal V$ exists. Let $\mathbf
R^{Card(\mathcal U)}$ have a norm $\|\cdot\|_{\mathbf
R^{Card(\mathcal U)}}$, with $\|\cdot\|_{\mathbf R^{Card(\mathcal
U)^*}}$ being its dual norm on $\mathbf
 R^{Card(\mathcal U)^*}$. Equipping $\mathcal W$ with the induced norm, and let
 $\|T_{\mathcal U}^{-1}\|:=\|T_{\mathcal U}^{-1}\|_{\mathcal W\rightarrow \mathcal
 V}.$  In addition, let
 $\mathcal K_+$ be the positive cone of $\mathbf R^{Card(\mathcal U)}$: that is, all
 $(r_u)\in\mathbf R^{Card(\mathcal U)}$ for which $r_u\geq0$. Then
 the following Lemma \ref{NORMING SET} can be found in  \cite{Mhaskar2000}.

\begin{lemma}\label{NORMING SET}
  Let $\mathcal U$ be a norm generating set for $\mathcal V$,
with $T_{\mathcal U}$ being the corresponding sampling operator. If
$v\in \mathcal V^*$ with $\|v\|_{\mathcal V^*}\leq A$, then there
exist real numbers $\{a_u\}_{u\in \mathcal Z}$, depending only on
$v$ such that for every $t\in  \mathcal V,$
$$
                       v(t)=\sum_{u\in \mathcal U}a_uu(t),
$$
and
$$
               \|(a_u)\|_{\mathbf R^{Card(\mathcal U)^*}}\leq A\|T_{\mathcal U}^{-1}\|.
$$
Also, if $\mathcal W$ contains an interior point $v_0\in \mathcal
K_+$ and if $v(T_{\mathcal U}^{-1}t)\geq0$ when $t\in \mathcal V\cap
\mathcal K_+$, then we may choose $a_u\geq 0.$
\end{lemma}

%The third lemma is a concentration inequality, which can be found in
%Theorem 11.6 of \cite{Gyorfi2002}.
%
%\begin{lemma}\label{CONCENTRATION2}
%Let $B\geq 1$ and let $\mathcal G$ be a set of functions $g:\mathbf
%S^d\rightarrow[0,R]$. Let $Z_1,\dots, Z_m$ be i.i.d. random
%variables. Assume $\alpha>0$, $0<\varepsilon<1$ and $m\geq 1$. Then,
%$$
%         \mathbf P^m\left\{\sup_{g\in B_R}\frac{\frac1m\sum_{i=1}^m
%         g(Z_i)-\mathbf E(g)}{\alpha+\frac1m\sum_{i=1}^m
%         g(Z_i)+\mathbf E(g)}>\varepsilon\right\}
%         \leq
%         4\mathbf E\mathcal
%         N_1\left(\frac{\alpha\varepsilon}5,\mathcal G,Z^m\right)
%         \exp\left\{-\frac{3\varepsilon^2\alpha m}{40R}\right\},
%$$
%where $\mathcal
%         N_1(\varepsilon,\mathcal G,Z^m)$
%denotes the covering number of $\mathcal G$ in the metric of
%$\mathcal l^1$ empirical norm.
%\end{lemma}

By the help of Lemma  \ref{BERNSTEIN INEQUALITY}, Lemma
\ref{NIKOLSKII INEQUALITY} and Lemma \ref{NORMING SET}
   we can deduce the following probabilistic
cubature formula.

\begin{proposition}\label{RANDOM CUBATURE ON SPHERE}
Let $N$ be a positive integer and $1\leq p\leq2$. If
$\Lambda_N:=\{t_i\}_{i=1}^N$ are i.i.d. random variables drawn
according to arbitrary distribution $\mu$ on $\mathbf S^d$, then
there exits a set of real numbers $\{a_i\}_{i=1}^N$ such that
$$
              \int_{\mathbf S^{d}}Q_n(x)d\omega(x)=\sum_{i=1}^Na_iQ_n(t_i)
$$
holds with confidence at least
$$
           1-2\exp\left\{-C\frac{ N }{D_{\rho_X}n^d}+Cn^d\right\},
$$
subject to
$$
                \sum_{i=1}^N|a_i|^p\leq\frac{|\mathbf S^d|}{1-\varepsilon}N^{1-p}.
$$
\end{proposition}

\begin{proof}
   Without loss of generality, we assume
$Q_n\in\mathcal P^0:=\{f\in\Pi_n^d:\|f\|_\rho\leq 1\}$. We denote
the $\delta$-net of all $f\in\mathcal P^0$, by $\mathcal A(\delta)$.
It follows from     \cite[Chap.9]{Gyorfi2002} and the definition of
the covering number that the smallest cardinality of $\mathcal
A(\delta)$ is bounded by
$$
       \exp\{Cn^d\log1/\delta\}.
$$

 Given
$Q_n\in\mathcal P^0$. Let $P_j$ be the polynomial in $\mathcal
A(2^{-j})$ which is closet to $Q_n$ in the uniform norm, with some
convention for breaking ties. Since $\|Q_n-P_j\|\rightarrow0$, with
the denotation $\eta_i(P)=|P(t_i)|^2-\|P\|_\rho^2,$ we can write
$$
       \eta_i(P)=\eta_i(P_0)+\sum_{l=0}^\infty\eta_i(P_{l+1})-\eta_i(P_l).
$$

Since   the sampling set $\Lambda_N$ consists of  a sequence of
i.i.d. random variables on $\mathbf S^{d}$, the sampling points are
a sequence of functions $t_j=t_j(\omega)$ on some probability space
$(\Omega,\mathbf P)$. If we set $\xi_j^2(P)=|P(t_j)|^2$, then
$$
           \eta_i(P)=|P(t_i)|^2-\|P\|_{\rho}^2
           =|P(t_i)(\omega)|^2-
             \mathbf E\xi_j^2,
$$
where we have used the equalities
$$
            \mathbf E\xi_j^2
            =\int_{\mathbf S^d}|P(x)|^2  d\rho_X=\|P\|_{\rho}^2.
$$
Furthermore,
$$
            |\eta_i(P)|\leq\sup_{\omega\in\Omega}
            \left||P( t_i(\omega))|^2-\|P\|_{\rho}^2\right|
            \leq \|P\|_{\infty}^2+\|P\|_{\rho}^2.
$$
It follows from Lemma \ref{NIKOLSKII INEQUALITY} that
$$
               \|P\|_{\infty}\leq Cn^{\frac{d}{2}}\|P\|_{2}.
$$
Hence
$$
             |\eta_i(P)-\mathbf E\eta_i(P)|\leq CD_{\rho_X}n^d.
$$
Moreover, using Lemma \ref{NIKOLSKII INEQUALITY} again, there holds,
$$
              \sigma^2(\eta_i(P))\leq\mathbf E((\eta_i(P))^2)\leq
              \|P\|_\infty^{2}\|P\|_{\rho}^{2}-\|P\|_{2}^{4}\leq CD_{\rho_X}n^d.
$$
Then, using Lemma \ref{BERNSTEIN INEQUALITY} with $\varepsilon=1/2$
and $M_\xi=\sigma^2=Cn^d$, we have for fixed $P\in \mathcal A(1)$,
with  probability at most  $2\exp\{-CN/D_{\rho_X}n^d\}$, there holds
$$
          \left|\frac1N\sum_{i=1}^N\eta_i\right|\geq\frac14.
$$
Noting there are at $\exp\{Cn^d\}$ polynomials in $\mathcal A(1)$,
we get
\begin{equation}\label{Probability1}
    \mathbf P^N\left\{ \frac1N\sum_{i=1}^N|\eta_i(N)\geq \frac14\ \mbox{  for some}\ P\in\mathcal
A(1)\right\}\leq
2\exp\left\{-\frac{CN}{D_{\rho_X}n^d}+Cn^d\right\}.
\end{equation}

Now, we aim to bound the probability of the event:

   ({\bf e}1) for some
$l\geq1$, some $P\in\mathcal A(2^{-l})$ and some $Q\in\mathcal
A(2^{-l+1})$ with $\|p-q\|\leq 3\times2^{-l}$, there holds
$$
       |\eta_i(P)-\eta_i(Q)|\geq\frac{1}{4(l+1)^2}.
$$

The main tool is also the Bernstein inequality. To this end, we
should bound $|\eta_i(P)-\eta_i(Q)-\mathbf E(\eta_i(P)-\eta_i(Q))$
and the variance $\sigma^2(\eta_i(P)-\eta_i(Q))$.  According to the
Taylor formula
$$
           a^2=b^2+(a+b)(a-b),
$$
and Lemma \ref{NIKOLSKII INEQUALITY},  we have
\begin{eqnarray*}
  \|\eta_i(P)-\eta_i(Q)\|
  &\leq&
  \sup_{\omega\in\Omega}\left||P(t_i(\omega))|^2-|Q(t_i(\omega))|^2\right|
  +
  |\|Q\|_\rho^2-\|P\|_\rho^2|\\
  &\leq&
  CD_{\rho_X}n^d\|P-Q\|,
\end{eqnarray*}
and
\begin{eqnarray*}
   \sigma^2(\eta_i(P)-\eta_i(Q))
   &\leq&
   \mathbf E((\eta_i(P)-\eta_i(Q))^2)\\
   &=&\int_{\mathbf
   S^d}(|P(x)|^2-|Q(x)|^2)^2d\rho_X-(\|P\|_\rho^2-\|Q\|_\rho^2)^2\\
   &\leq&
   CD_{\rho_X}n^d\|P-Q\|^2.
\end{eqnarray*}
If $P\in\mathcal A(2^{-l})$ and $Q\in\mathcal A(2^{-l+1})$ with
$\|P-Q\|\leq  3\times2^{-l}$, then it follows from Lemma
\ref{BERNSTEIN INEQUALITY} again that,
\begin{eqnarray*}
    \mathbf
    P^N\left(\left|\sum_{i=1}^N\eta_i(P)-\eta_i(Q)\right|>\frac1{4(l+1)^2}\right)
     &\leq&
    2\exp\left\{-\frac{N}{CD_{\rho_X}n^d(2^{-2l}l^4+2^{-l}l^2)}\right\}\\
    &\leq&
    2\exp\left\{-\frac{N}{CD_{\rho_X}n^d 2^{-l/2}}\right\}
\end{eqnarray*}
Since there are at most $2\exp\{-Cn^d\log l\}$ polynomials in
$\mathcal A(2^{-l})\cup\mathcal A(2^{-l+1})$, then the event ({\bf
e}1) holds with probability at most
$$
       \sum_{l=1}^\infty2\exp\left\{-\frac{CN}{ D_{\rho_X}n^d
       2^{-l/2}}+Cn^d\log l\right\}
       \leq
       \sum_{l=1}^\infty2\exp\left\{-2^{l/2}\left(\frac{C N}{ D_{\rho_X}n^d
      }-n^d\right)\right\}.
$$
Since $\sum_{i=1}^\infty e^{-a^lb}\leq Ce^{-b}$ for any $a>1$ and
$b\geq 1$, we then deduce that
\begin{equation}\label{Probability2}
        \mathbf P^m \{\mbox{The event ({\bf e}1) holds} \}\leq
        2\exp\left\{\frac{CN}{D_{\rho_X}n^d}-Cn^d\right\}.
\end{equation}
Thus, it follows from (\ref{Probability1}) and (\ref{Probability2})
that with confidence at least
$$
         1-2\exp\left\{\frac{CN}{D_{\rho_X}n^d}-Cn^d\right\}
$$
there holds
\begin{eqnarray*}
       \left|\sum_{i=1}^n\eta_i(Q_n)\right|
       &\leq&
       \left|\sum_{i=1}^n\eta_i(P_0)\right|+
       \sum_{l=1}^\infty|\sum_{i=1}^n\eta_i(P_l)-\eta_i(P_l)|\\
       &\leq&
       \frac14+\sum_{l=1}^\infty\frac1{4(l+1)^2}
       =
       \sum_{l=1}^\infty\frac1{4l^2}=\frac{\pi^2}{24}<\frac12.
\end{eqnarray*}
This means that with confidence at least
$$
         1-2\exp\left\{\frac{C N}{D_{\rho_X}n^d}-Cn^d\right\}
$$
there holds
\begin{equation}\label{MZ}
           \frac12\|Q_n\|_{\rho}^2
           \leq
           \frac1N\sum_{i=1}^N|Q_n(\alpha_i)|^2\leq\frac32\|Q_n\|_{\rho}^2\quad
           \forall Q_n\in\Pi_n^{d}.
\end{equation}

Now, we use (\ref{MZ}) and Lemma \ref{NORMING SET} to prove Lemma
\ref{RANDOM CUBATURE ON SPHERE}. In Lemma \ref{NORMING SET}, we take
$\mathcal V=\Pi_n^{d}$, $\|Q_n\|_{\mathcal V}=\|Q_n\|_{\rho}$, and
$\mathcal W$ to be the set of point evaluation functionals
$\{\delta_{t_i}\}_{i=1}^N$. The operator $T_{\mathcal W}$ is then
the restriction map $Q_n\mapsto Q_n|_{\Lambda},$ with
$$
            \|f\|_{\Lambda,2}^2:=
             \left(\frac1N\sum_{i=1}^N|f(t_i)|^p\right)^\frac12.
$$
 It follows from
(\ref{MZ})   that with confidence at least
$$
         1-2\exp\left\{\frac{CN}{D_{\rho_X}n^d}-Cn^d\right\}
$$
there holds $\|T_{\mathcal W}^{-1}\|\leq2. $ We now take $u$ to be
the functional
$$
              y: Q_n \mapsto \int_{\mathbf S^{d}}Q_n(x) d\rho_X.
$$
By H\"{o}lder inequality, $\|y\|_{\mathcal V^*}\leq 1$. Therefore,
Lemma \ref{NORMING SET} shows that
$$
              \int_{\mathbf
              S^{d}}Q_n(x)d\omega(x)=\sum_{i=1}^Na_iQ_n(t_i)
$$
holds with confidence at least
$$
         1-2\exp\left\{\frac{C N}{D_{\rho_X}n^d}-Cn^d\right\}
$$
subject to
$$
                \frac1N\sum_{i=1}^N\left(\frac{|a_i|}{1/N}\right)^2
                \leq2.
$$
Then,   the H\"{o}lder   finishes the proof of Proposition
\ref{RANDOM CUBATURE ON SPHERE}.
\end{proof}

 To estimate the upper bound of
 $$          \mathcal E(\pi_Mf_{{\bf z},\lambda,q})-\mathcal
 E(f_\rho),
 $$
 we first
introduce an error decomposition strategy. It follows from the
definition of $f_{{\bf z},\lambda,q}$ that, for arbitrary
$f\in\mathcal H_{K,{\bf z}}$,
\begin{eqnarray*}
            \mathcal E( \pi_Mf_{{\bf z},\lambda,q})-\mathcal E(f_\rho)
            &\leq&
            \mathcal E( \pi_Mf_{{\bf z},\lambda,q})-\mathcal
            E(f_\rho)+\lambda\Omega_{\bf z}^q(f_{{\bf
            z},\lambda,q})\\
            &\leq&
            \mathcal E( \pi_Mf_{{\bf z},\lambda,q})-\mathcal E_{\bf z}( f_{{\bf
            z},\lambda,q})+\mathcal E_{\bf z}(f)-\mathcal E(f)\\
            &+&
            \mathcal E_{\bf z}( \pi_Mf_{{\bf z},\lambda,q})+\lambda\Omega_{\bf z}^q(\pi_Mf_{{\bf
            z},\lambda,q})-\mathcal E_{\bf z}(f)-\lambda\Omega_{\bf
            z}^q(f)\\
            &+&
            \mathcal E(f)-\mathcal E(f_\rho)+\lambda\Omega_{\bf
            z}^q(f)\\
             &\leq&
             \mathcal E(\pi_M f_{{\bf z},\lambda,q})-\mathcal E_{\bf z}(\pi_M f_{{\bf
            z},\lambda,q})+\mathcal E_{\bf z}(f)-\mathcal E(f)\\
            &+&
            \mathcal E(f)-\mathcal E(f_\rho)+\lambda\Omega_{\bf
            z}^q(f).
\end{eqnarray*}

Since $f_\rho\in W_r$ with $r>\frac{d}2$, it follows from the
Sobolev embedding theorem and Jackson inequality \cite{Brown2005}
that
   there exists a
$P_\rho\in\Pi_n^d$ such that
\begin{equation}\label{app1}
             \|P_\rho\|\leq c\|f_\rho\| \quad\mbox{and}\quad
             \|f_\rho-P_\rho\|^2\leq Cn^{-2r}.
\end{equation}
Then we have
\begin{eqnarray*}
            \mathcal E(f_{{\bf z},\lambda,q})-\mathcal E(f_\rho)
            &\leq&
            \left\{\mathcal E(P_\rho)-\mathcal E(f_\rho)+\lambda\Omega_{\bf
            z}^q(P_\rho)\right\}\\
            &+&
             \left\{\mathcal E(f_{{\bf z},\lambda,q})-\mathcal E_{\bf z}(f_{{\bf
            z},\lambda,q})+\mathcal E_{\bf z}(P_\rho)-\mathcal
            E(P_\rho)\right\}\\
            &=:&
            \mathcal D({\bf z},\lambda,q)+\mathcal S({\bf
            z},\lambda,q),
\end{eqnarray*}
where $\mathcal D({\bf z},\lambda,q)$ and $\mathcal S({\bf
z},\lambda,q)$ are called as the approximation error and  sample
error, respectively. The following Proposition \ref{APPROXIMATION
ERROR} presents an upper bound for the approximation error.

\begin{proposition}\label{APPROXIMATION ERROR}
Let $m,n\in\mathbf N$, $r>d/2$ and $f_\rho\in W_r$. Then, with
confidence at least $1-2\exp\{{-cm/(D_{\rho_X}n^d)}\},$ there holds
$$
          \mathcal D({\bf
          z},\lambda,q)\leq
          C\left(n^{-2r}+2\lambda m^{1-q}\right),
$$
where $C$ and $c$ are constants depending only on $d$ and $r$.
\end{proposition}

\begin{proof}
  From Lemma \ref{PROPERTY}, it
is easy to deduce that
$$
              P_\rho(x)=\int_{\mathbf S^d}P_\rho(x')K_{n}(x,x')d\omega(x').
$$
Thus, Proposition \ref{RANDOM CUBATURE ON SPHERE}, H\"{o}lder
inequality
 and $r>d/2$ yield  that with confidence at least
$1-2\exp\{{-cm/n^d}\},$ there exists a set of real numbers
$\{a_i\}_{i=1}^m$ satisfying $\sum_{i=1}^m|a_i|^q\leq 2m^{1-q}$ for
$q>0$ such that
$$
            P_\rho(x)=\sum_{i=1}^ma_iP_\rho(x_i)K_n(x_i,x).
$$
The above observation together with (\ref{app1}) implies that with
confidence at least $1-2\exp\{{-cm/(D_{\rho_X}n^d)}\},$  $P_\rho$
can be represented as
$$
        P_\rho(x)=\sum_{i=1}^ma_iP_\rho(x_i)K_n(x_i,x)\in\mathcal
         H_{K,{\bf z}}
$$
such that for arbitrary $f_\rho\in W_r$, there holds
$$
             \|P_\rho-f_\rho\|_\rho^2\leq\|P_\rho-f_\rho\|^2\leq Cn^{-2r},
$$
and
$$
         \Omega_{\bf z}^q(P_\rho)\leq\sum_{i=1}^m|a_iP_\rho(x_i)|^q\leq
         (cM)^q\sum_{i=1}^m|a_i|^q\leq
         2|\mathbf S^d|m^{1-q},
$$
where $C$ is a constant depending only on $d$ and $M$.
%Indeed, if
%$q\geq1$, we have $\sum_{i=1}^m|a_i|^q\leq 2|\mathbf S^d| m^{1-q}$.
%If $0< q<1$, it follows from  the H\"{o}lder inequality that
%$$
%         \sum_{i=1}^m|a_i|^q\leq\left(\sum_{i=1}^m|a_i|\right)^q\left(\sum_{i=1}^m1\right)^{1-q}\leq
%         m^{1-q}(2\Omega_{d})^q\leq 2|\mathbf S^d|m^{1-q}.
%$$
It thus implies that the inequalities
\begin{equation}\label{approximation error}
          \mathcal D({\bf
          z},\lambda,q)\leq\|P_\rho-f_\rho\|_\rho^2+\lambda\Omega_{\bf
          z}^q(g^*)\leq
          C\left(n^{-2r}+2\lambda m^{1-q}\right)
\end{equation}
holds with confidence at least $1-2\exp\{{-cm/(D_{\rho_X}n^d)}\}.$
\end{proof}

At last, we  deduce the final learning rate of $l^q$ kernel
regularization schemes (\ref{algorihtm1}). Firstly, it follows from
Propositions \ref{APPROXIMATION ERROR}, \ref{BOUND S1} and
\ref{BOUND S2} that
\begin{eqnarray*}
            \mathcal
            E( \pi_Mf_{{\bf
           z},\lambda,q})-\mathcal E(f_\rho))
           &\leq&
            \mathcal D({\bf z},\lambda,q)+\mathcal S_1^q+\mathcal
            S_2^q
           \leq
           C\left(n^{-2r}+\lambda m^{1-q}\right)\\
           &+&
           \frac12(\mathcal
            E(f_{{\bf
           z},\lambda,q})-\mathcal E(f_\rho))+2\varepsilon
\end{eqnarray*}
holds with confidence at least
\begin{eqnarray*}
           1-4\exp\{{-cm/(D_{\rho_X}n^d)}\}
         -\exp\left\{cn^d\log\frac{4M^2}{\varepsilon}
              -\frac{3m\varepsilon}{128M^2}\right\}
          -\exp\left(-\frac{3m\varepsilon^2}{48M^2\left(2n^{-2r}
            +\varepsilon\right)}\right).
\end{eqnarray*}

Then,  by setting  $\varepsilon\geq\varepsilon_m^+\geq C(m/\log
m)^{-2r/(2r+d)}$, $n=\left[\varepsilon^{-1/(2r)}\right]$ and
$\lambda\leq m^{q-1}\varepsilon$,     it follows from $r>d/2$ that
           \begin{eqnarray*}
          &&1-5\exp\{-CD_{\rho_X}^{-1}m\varepsilon^{d/(2r)}\}-\exp\{-Cm\varepsilon\}\\
          &-&
          \exp\left\{C\varepsilon^{-d/(2r)}\left(\log1/\varepsilon+\log m\right)
          -Cm\varepsilon)\right\}\\
          &\geq&
          1-6\exp\{-Cm\varepsilon\}.
\end{eqnarray*}
That is, for  $\varepsilon\geq \varepsilon_m^+$,
$$
            \mathcal
            E(f_{{\bf
           z},\lambda,q})-\mathcal E(f_\rho)
           \leq
           6\varepsilon
$$
holds with confidence at least
$1-6\exp\{-CD_{\rho_X}^{-1}m\varepsilon\}$. The same method as
\cite[P.37]{Devore2006} and the fact that the uniform distribution
satisfies $D_{\rho_X}<\infty$ yields the lower bound of
(\ref{Theorem 2}). This finishes the proof of Theorem \ref{THEOREM
2}.

\section{Conclusion and discussion}
Since its inception in \cite{Narcowich2006}, needlets have become
the most popular tools to tackle spherical data due to its perfect
localization performance in both the frequency and spacial domains.
The main novelty  of the present paper is to suggest the usage of
the needlet kernel in kernel methods to deal with spherical data.
Our contributions   can be summarized as follows. Firstly, the model
selection problem of the kernel ridge regression boils down to
choosing a suitable kernel and the corresponding regularization
parameter. Namely, there are totally two types parameters in the
kernel methods. This requires relatively large amount of
computations when faced with large-scaled data sets. Due to needlet
kernel's excellent localization property in the frequency domain, we
prove that, if a truncation operator is added to the final estimate,
then as far as the model selection is concerned, the regularization
parameter is not necessary in the sense of rate optimality. This
means that there is only a discrete parameter, the frequency of the
needlet kernel, needs tuning  in the learning process, which
presents a theoretically guidance to reduce the computation burden.
Secondly, Compared with the kernel ridge regression, $l^q$ kernel
regularization learning, including the kernel lasso estimate and
kernel bridge estimate, may bring  a certain additional attribution
of the estimator, such as the sparsity. When utilized the $l^q$
kernel regularization learning, the focus is to judge whether it
degrades the generalization capability of the kernel ridge
regression. Due to needlet kernel's excellent localization property
in the spacial domain,  we have proved in this paper that, on the
premise of embodying the feature of   the $l^q$ $ (0< q\leq 2)$
kernel regularization learning, the selection of $q$ doesn't affect
the generalization error  in the sense of rate optimality. Both of
them showed  that the needlet kernel is an good choice of the kernel
method to deal with spherical data.

We conclude this paper with the following important remark.

\begin{remark}
There are two types of polynomial kernels for spherical data
learning: the localized kernels and non-localized kernels. For the
non-localized kernels,  there are three papers focused on its
applications in nonparametric regression. \cite{Minh2006} is the
first one to derive the learning rate of KRR associated with the
polynomial kernel $(1+x \cdot x')^n$. However their learning rate
were built upon the assumption that $f_\rho$ is a polynomial.
\cite{Li2009} omitted this assumption  by using the eigenvalue
estimate of the polynomial kernel. But the derived learning rate of
\cite{Li2009} is not optimal. \cite{Cao2013} conducted a learning
rate analysis for KRR associated the reproducing kernel of the space
$(\Pi_n^d,L_2(\mathbf S^d))$ and derived the similar learning rate
as \cite{Li2009}. In a nutshell, for the spherical data learning, to
the best of our knowledge, there didn't exist almost   optimal
minimax learning rate analysis for KRR associated with non-localized
kernels. Using the methods in the present paper, especially the
technique in bounding the sampling error, we can improve the results
in \cite{Cao2013} and \cite{Li2009} to the almost optimal minimax
learning rates. For the localized kernels, such as the kernels
proposed in \cite{Brown2005,Filbir2004,LeGia2008,Mhaskar2005}, we
can derive similar results as the needlet kernel in this paper. That
is, the almost optimal learning rates of KRR and $l_q$ KRS can be
derived for these kernels by using the same method in the paper.
Since needlets' popularity in statistics and real world
applications, we only present the learning rate analysis for the
needlet kernel. Finally, it should be pointed out that when
$y_i=f_\rho(x_i)$, the learning rate of the least squares (KRR with
$\lambda=0$) associated with a localized kernel was derived in
\cite{LeGia2008}. The most important difference between our paper
and \cite{LeGia2008} is we are faced with nonparametric regression
problem, while \cite{LeGia2008} focused on the approximation
problems.

\end{remark}

\end{document}